\documentclass[letterpaper]{article}
\usepackage[preprint]{aaai2027}
\usepackage[hyphens]{url}
\usepackage{graphicx}
\usepackage{natbib}
\usepackage{caption}
\usepackage{booktabs}
\usepackage{array}
\usepackage{amsfonts}
\usepackage{amsmath}
\usepackage{amssymb}
\usepackage{amsthm}
\usepackage{xcolor}
\newtheorem{proposition}{Proposition}
\newtheorem{corollary}{Corollary}
\newtheorem{claim}{Claim}
\newtheorem{assumption}{Assumption}

\newcolumntype{L}[1]{>{\raggedright\arraybackslash}p{#1}}

\title{Decision-Metric Alignment in Latent World Models:
Diagnostics and Action-Conditioned Objectives for MPC Planning}

\author{
Jiawei Wang$^{1}$%
\thanks{Equal contribution.}%
\thanks{Corresponding author 
(\texttt{jarvisustc@gmail.com}).},
Yushen Zuo$^{1}$\footnotemark[1],
Ke Rui$^{1}$\footnotemark[1],
Yichun Feng$^{2}$,
Minglei Li$^{1}$\footnotemark[2]
}

\affiliations{
$^{1}$Simple AI, Beijing, China\\
$^{2}$University of Chinese Academy of Sciences, Beijing, China
}

\begin{document}

\maketitle

\begin{abstract}
JEPA-style latent world models can use Euclidean distance to a goal
latent as the cost for model-predictive control (MPC). Strong decoding
of task variables, however, does not guarantee that this particular
cost ranks candidate action sequences by real task progress. We call
the latter property \emph{decision-metric alignment}. We introduce
Plan-Real Spearman, which measures latent--real rank agreement on random
plans, and CEM-stage Spearman, which measures the same agreement as
cross-entropy-method (CEM) search concentrates its proposal. We analyze
sufficient conditions under which latent distance preserves real-cost
rankings, identifying encoder distortion, terminal rollout error, and
candidate margins as the controlling quantities. Guided by the observed
empirical alignment gap, DA-LeWM augments LeWM
with inverse-dynamics and demonstration-conditioned goal-action heads.
Across all our experiments, DA-LeWM accelerates convergence and achieves
higher online success than LeWM, while probe scores remain similar. These
results show that action-conditioned objectives improve the geometry used
by Euclidean-cost, CEM-based latent MPC.
\end{abstract}

\section{Introduction}
\label{sec:intro}

Model predictive control on a learned latent world model is a
principled recipe for robotic manipulation: roll out the world model
along candidate action sequences, select the lowest-cost sequence,
execute, and replan~\citep{lewm2026, ha2018world, hafner2019planet,
hafner2019dreamer, hansen2024tdmpc2}. When the
world model operates in a JEPA-style learned latent
space~\citep{lecun2022jepa, assran2023ijepa}, the planning cost is
typically the squared Euclidean distance between the predicted latent
trajectory and an encoded goal observation, dispensing with the need
for an explicit reward~function.

This recipe exhibits a striking pattern: small changes to the training
objective induce large and often counterintuitive changes in online
planning success, while linear probes for state, action, reward, and
value remain nearly unchanged across the same variants. The dominant
question asked of a latent world model, ``what does the latent
encode?'', is therefore incomplete for planning.

The missing property is geometric rather than informational. The
planner's cost depends on Euclidean distances between latents, so for
the ranking of candidate plans to track real-world outcomes, the
latent's metric structure must be compatible with the structure of
action consequences in the environment. We formalize this as the
distinction between \emph{information sufficiency} (whether
task-relevant quantities can be decoded from the latent) and
\emph{decision-metric alignment} (whether the latent cost ranks
candidate plans consistently with their environmental outcomes). The
two are independent: as Figure~\ref{fig:concept} illustrates, a latent
can be information-sufficient yet geometrically misaligned.

\begin{figure*}[t]
  \centering
  \includegraphics[width=0.94\textwidth]{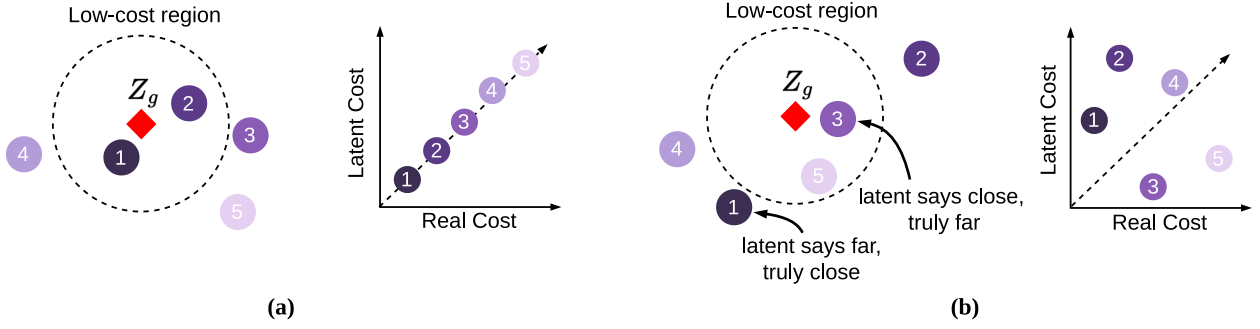}
  \caption{\textit{Information sufficiency does not imply
  decision-metric alignment.} In each panel, the left schematic places
  candidate plans around the goal latent $z_g$, while the right plot
  compares their real costs (horizontal) with their latent Euclidean
  costs (vertical). In (a), the two costs induce the same ordering. In
  (b), candidate 1 is truly low-cost but lies far from $z_g$, whereas
  candidate 3 has higher real cost but lies close. Latent distance
  therefore reverses their preference even though candidate identity
  and rank remain decodable.}
  \label{fig:concept}
\end{figure*}

To make decision-metric alignment testable, we introduce two
diagnostics. \emph{Plan-Real Spearman} measures the rank correlation
between latent costs and real costs over random candidate plans.
\emph{CEM-stage Spearman} extends it to the converged elite candidates
the agent actually executes. We analyze sufficient conditions under which
latent distance preserves real-cost rankings
(Section~\ref{sec:analysis}). The analysis identifies encoder distortion,
terminal rollout error, and candidate margins as the controlling quantities,
motivating stage-wise evaluation as CEM concentrates its proposal. We treat
SIGReg measurements as empirical collapse diagnostics, not as certificates
of the analytical encoder constants. Guided by the observed empirical
alignment gap, we then study action-conditioned auxiliary supervision through
\textbf{DA-LeWM} (Decision-Aligned LeWM), which augments LeWM with
two lightweight auxiliary heads, an inverse-dynamics head and a
goal-conditioned action head. Across all our experiments, DA-LeWM
accelerates convergence and achieves higher online success than LeWM,
while linear-probe scores remain similar.

Our key contributions are as follows:
\begin{itemize}
\item We distinguish information sufficiency from decision-metric
  alignment and introduce Plan-Real Spearman and CEM-stage Spearman
  to measure the latter.
\item We analyze sufficient conditions for rank preservation, identifying
  encoder distortion, terminal rollout error, and candidate margins as
  the controlling quantities.
\item Guided by the observed alignment gap, we introduce DA-LeWM with
  action-conditioned auxiliary supervision. It improves latent--real rank
  agreement in our geometry diagnostics and, across all four environments,
  accelerates convergence and achieves higher online success than LeWM
  despite similar linear-probe scores.
\end{itemize}

\section{Background and Problem Setup}
\label{sec:background}

\subsection{JEPA Latent World Models}

We build on LeWM \citep{lewm2026}, a JEPA-style world model for robotic
manipulation. An encoder $f_\theta$ maps an observation $o_t$ to a latent
embedding $z_t = f_\theta(o_t)$ via a Vision Transformer backbone followed
by a linear projector. An autoregressive predictor $g_\phi$ estimates
future latents conditioned on actions:
\begin{equation}
\hat{z}_{t+1} = g_\phi(z_{t-H:t},\; a_{t-H:t}).
\end{equation}
Training minimizes a prediction loss together with Sketched Isotropic
Gaussian Regularization (SIGReg), introduced by LeJEPA
\citep{balestriero2025lejepa} and adopted by LeWM:
\begin{equation}
\mathcal{L}_\text{WM} = \mathcal{L}_\text{pred} + \lambda_\text{sig}
\mathcal{L}_\text{SIGReg}.
\end{equation}
SIGReg projects embeddings onto random unit directions and
minimizes the one-dimensional Epps--Pulley statistic against a standard
Gaussian, encouraging the aggregate embedding distribution to match an
isotropic Gaussian and discouraging trivial collapse. This is a
distribution-level target: it does not bound encoder-Jacobian singular
values or imply local metric isotropy, injectivity, or bi-Lipschitz
behavior.

\subsection{MPC with Latent Goal Distance}

At test time, MPC is realized by the cross-entropy method (CEM)
\citep{rubinstein1999cem}. Given an initial observation $o_0$ and a goal
observation $o_g$, the planner solves
\begin{equation}
a^* = \arg\min_{a_{0:H}} \|\hat{z}_H - z_g\|^2,
\quad z_g = f_\theta(o_g),
\end{equation}
where $\hat{z}_H$ is obtained by rolling out $g_\phi$ under candidate
action sequences. CEM iteratively refits its proposal distribution to the
top-$E$ elite candidates with lowest latent cost.

\paragraph{Goal acquisition at evaluation time.}
For evaluation, $o_g$ is drawn from a held-out demonstration trajectory
at a fixed offset $\Delta$ ahead of the initial observation. This
guarantees that the goal latent lies in the data distribution and is
reachable by some action sequence. In deployment, the goal image would
typically come from a single human demonstration or sub-goal generation.
Our analysis is independent of the goal-acquisition mechanism.

\subsection{Auxiliary Objectives}

A standard avenue for improving latent representations is to add
auxiliary prediction heads jointly trained with the world model
\citep{schwarzer2021spr, laskin2020curl, gelada2019deepmdp}. We
consider inverse-dynamics, goal-conditioned action, reward-proxy, and
value-proxy heads (full losses are in Supplementary Section~A).
Such objectives are usually judged by probes or downstream control. We
instead ask how they affect the planner's fixed Euclidean cost.

\section{Information Sufficiency vs.\ Decision-Metric Alignment}
\label{sec:framework}

\subsection{Definitions}

A latent space $z=f_\theta(o)$ is \emph{$\epsilon$-information-sufficient}
for task quantities $\{q_k\}$ if probe-specific regressors
$\hat q_k=h_k(x_k)$ attain held-out error at most $\epsilon$.
We test state, inverse-action, reward, value, and goal-action with linear
probes. Supplementary Section~A specifies each input $x_k$.

A latent space is \emph{decision-metric aligned} with respect to a planning
cost $c_\text{lat}$ if, for any pair of candidate action sequences
$\mathbf{a}^{(i)}, \mathbf{a}^{(j)}$ with corresponding real-environment
costs $c_\text{real}^{(i)}, c_\text{real}^{(j)}$,
\begin{equation}
c_\text{lat}(\mathbf{a}^{(i)}) < c_\text{lat}(\mathbf{a}^{(j)})
\quad\Longleftrightarrow\quad
c_\text{real}(\mathbf{a}^{(i)}) < c_\text{real}(\mathbf{a}^{(j)}).
\end{equation}

Decision-metric alignment is an ordinal property, not a numerical one. It
requires the latent cost to correctly rank candidates rather than to be
numerically equal to any reward function. The two properties are
logically independent. As Figure~\ref{fig:concept} illustrates, a
latent can encode all task-relevant information while arranging it in
a geometry where Euclidean distance ranks candidates poorly.

\subsection{Plan-Real and CEM-stage Spearman}
\label{sec:plan-real}

For each (start, goal) pair sampled from a held-out dataset with
offset $\Delta$, we draw $N=64$ candidate action sequences from the
CEM proposal distribution, score each candidate by the world model's
latent cost $c_\text{lat}^{(i)} = \|\hat{z}_H^{(i)} - z_g\|^2$ and by
environment-rollout cost $c_\text{real}^{(i)} = d_\text{task}(s_H^{(i)},
s_g)$, and compute the Spearman rank correlation between the two
across the $N$ candidates. On PushT, $d_\text{task}=\|s_H-s_g\|_2$
over all seven simulator coordinates (agent and block pose, and agent
velocity), rather than thresholded online success. \emph{Plan-Real
Spearman} is the mean over $n = 30$ pairs. Spearman correlation lies in $[-1,1]$: $+1$ denotes
identical rankings, $0$ no monotone association, and $-1$ reversed
rankings. Exact ties receive their average rank. If either cost vector
is constant, the pair-level correlation is undefined. We exclude that
pair from the mean and record the number of defined pairs. Ranking is
the relevant quantity because CEM updates its
proposal from relative candidate order rather than calibrated costs.

CEM does not act on random plans: it iteratively concentrates
candidates toward low-cost regions, and a latent cost that ranks
random plans well may still fail when candidates are near-optimal.
\emph{CEM-stage Spearman} measures the same correlation at three CEM
stages: random (iteration $0$), mid (iteration $15$), and elite (the
final top-$E$ set). Random-stage agreement tests global coarse ordering.
Mid-stage agreement tests the ordering that shapes proposal updates.
Elite-stage agreement tests local discrimination among near-optimal
candidates.

\begin{figure*}[t]
  \centering
  \includegraphics[width=0.98\textwidth]{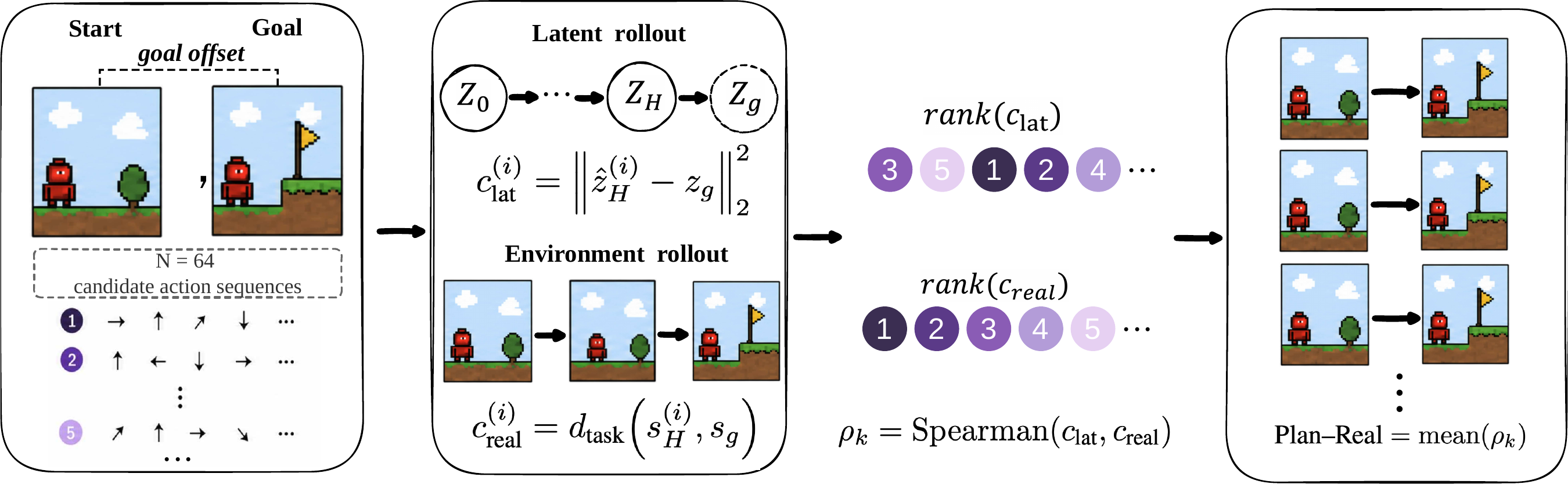}
  \caption{\emph{Plan-Real Spearman measurement procedure.} This operationalizes the
  candidate-order comparison in Figure~\ref{fig:concept}. For each
  held-out pair $k$, the same $N=64$ action sequences are evaluated by
  the world model and the environment, yielding paired latent- and
  real-cost vectors. Spearman gives a pair-level rank correlation $\rho_k$;
  Plan--Real averages the defined correlations among $n=30$ sampled pairs.
  CEM-stage Spearman reuses the paired
  scoring procedure with candidates from the selected CEM population.}
  \label{fig:plan-real}
\end{figure*}

\section{DA-LeWM: A Decision-Aligned Latent World Model}
\label{sec:method}

\subsection{Method}

We propose \textbf{DA-LeWM}, which augments the base LeWM training with
two lightweight auxiliary heads. The full training objective is
\begin{equation}
\mathcal{L}_\text{DA-LeWM} = \mathcal{L}_\text{WM}
+ \alpha\, \mathcal{L}_\text{inv}
+ \beta\, \mathcal{L}_\text{goal},
\end{equation}
where
$\mathcal{L}_\text{inv} = \|h_\text{inv}(z_t, z_{t+1}) - a_t\|^2$
is the inverse-dynamics loss and
$\mathcal{L}_\text{goal} = \|h_\text{goal}(z_t, z_g) - a_t\|^2$
is the goal-conditioned action loss. Both heads are small MLPs with
two hidden layers of width $256$. We set $\alpha = 0.1$ and
$\beta = 0.1$ as the cross-environment default. Sensitivity to $\beta$
is reported in Supplementary Section~C.

\paragraph{Goal-action supervision.}
Training uses four-frame demonstration clips at raw offsets
$0,5,10,15$. For $k\in\{0,1,2\}$, $h_\text{goal}(z_{t+5k},z_{t+15})$ is
supervised by the concatenated, normalized five-raw-action block
$a_{t+5k:t+5k+4}$. The goal is $15,10,$ or $5$ raw steps ahead.
The target reflects demonstrated progress, not an optimizer-computed
globally optimal action. With multimodal demonstrations, squared-error
training may average incompatible actions. The head shapes the
representation and is not used as a policy.

\paragraph{Rationale.}
The inverse-dynamics loss encourages the latent transition to retain
information predictive of $a_t$.
The goal-conditioned action head provides analogous structure relative
to the goal latent. An exploratory all-heads checkpoint using the state-derived proxies
in Supplementary Section~A is retained as an implementation-specific
comparison. Because those proxies are not the environment's reward and
value targets, this checkpoint neither isolates an R/V effect nor tests
correctly specified reward- or value-aware training in general.

\subsection{Sufficient-Condition Analysis}
\label{sec:analysis}

We complement the empirical diagnostics with an analytical account of
when latent-cost ranking can be expected to track real-cost ranking. The
goal is to identify the principal error sources, not to prove a tight
rank-correlation bound.

\paragraph{Setup and assumptions.}

We use the notation of Section~\ref{sec:background}. For compactness,
$f_\theta(s)$ denotes encoding the rendered observation $o=R(s)$. The
encoder does not consume simulator state directly. Let $T$ denote the
environment dynamics, and write $s_H(\mathbf{a})$ and $\hat{z}_H(\mathbf{a})$
for the real and predicted terminal states under plan $\mathbf{a}$. Real
and latent costs are $c_\text{real}(\mathbf{a}) = \|s_H(\mathbf{a}) -
s_g\|_\Sigma$ and $c_\text{lat}(\mathbf{a}) = \|\hat{z}_H(\mathbf{a}) -
z_g\|_2$, respectively, for some PSD~metric~$\Sigma$. The planner
uses the squared latent norm, but positive squaring preserves every
candidate ranking, so the analysis uses the unsquared norm. The PushT
diagnostics instantiate $\Sigma=I$ on the full seven-dimensional state.

\begin{assumption}[Pointwise terminal-rollout consistency]
\label{ass:pred}
For every horizon-$H$ candidate plan in the analysis, terminal rollout
error satisfies
$\|\hat{z}_H(\mathbf{a})-f_\theta(s_H(\mathbf{a}))\|_2
\leq \epsilon_H$.
\end{assumption}

\begin{assumption}[Encoder bi-Lipschitz]
\label{ass:bilip}
There exist constants $0 < \mu_f \leq L_f$ such that, for the
terminal-state/goal pairs induced by candidate plans in the analysis,
\[
\mu_f \|s - s'\|_\Sigma \;\leq\; \|f_\theta(s) - f_\theta(s')\|_2
\;\leq\; L_f \|s - s'\|_\Sigma.
\]
\end{assumption}

SIGReg empirically discourages representational collapse, but by itself
does not establish Assumption~\ref{ass:bilip} or a lower bound on
$\mu_f$. We therefore state bi-Lipschitz behavior as an assumption and
measure collapse-related surrogates separately. Training losses are
empirical diagnostics, not proofs that these assumptions hold off
distribution.

\paragraph{Pointwise cost-approximation bound.}

\begin{proposition}[Two-sided cost-approximation bound]
\label{prop:sufficient}
Under Assumptions~\ref{ass:pred} and \ref{ass:bilip}, for any
horizon-$H$ plan $\mathbf{a}$,
\begin{equation}
\mu_f\, c_\text{real}(\mathbf{a}) - \epsilon_H
\;\leq\;
c_\text{lat}(\mathbf{a})
\;\leq\;
L_f\, c_\text{real}(\mathbf{a}) + \epsilon_H.
\label{eq:pointwise}
\end{equation}
\end{proposition}

The bracket width $(L_f - \mu_f) c_\text{real} + 2\epsilon_H$
decomposes into an \emph{encoder distortion term} and a
\emph{predictor error term}. Our experiments measure collapse-related
and rollout-error diagnostics, but do not identify these constants.

\begin{proof}[Proof sketch]
The reverse triangle inequality and Assumption~\ref{ass:pred} give
$\big|c_\text{lat} - \|f_\theta(s_H) - z_g\|_2\big| \leq \epsilon_H$.
Assumption~\ref{ass:bilip} (and $z_g\!=\!f_\theta(s_g)$) gives
$\|f_\theta(s_H) - z_g\|_2 \in [\mu_f c_\text{real}, L_f c_\text{real}]$.
\end{proof}

\begin{corollary}[Margin-implied rank preservation]
\label{cor:rank}
Under the assumptions of Prop.~\ref{prop:sufficient}, for any pair
with $c_\text{real}^{(i)} < c_\text{real}^{(j)}$,
$\mu_f c_\text{real}^{(j)} > L_f c_\text{real}^{(i)} + 2\epsilon_H$
implies $c_\text{lat}^{(i)} < c_\text{lat}^{(j)}$. If every pair of
$N$ distinct candidates satisfies this, Plan-Real Spearman is $+1$.
(Proof: subtract the upper bound at $i$ from the lower bound at~$j$.)
\end{corollary}

The condition requires (i) a low-distortion encoder
($L_f/\mu_f\!\to\!1$) and (ii) non-vanishing
candidate margins
$\Delta_{ij}\!=\!c_\text{real}^{(j)}\!-\!c_\text{real}^{(i)}$. The
latter can be small on contact-rich tasks under random sampling and can
shrink further as CEM concentrates its proposal.  This motivates
measuring rank agreement at random, intermediate, and elite stages
rather than assuming that one sampling regime is representative.

\paragraph{Empirical consistency check.}
Fit $c_\text{lat}\!\approx\!\kappa c_\text{real}$ and let
$\eta_i\!=\!|c_\text{lat}^{(i)}-\kappa c_\text{real}^{(i)}|$.
Kendall's $\tau_a$ averages pairwise concordance signs (ties count as
zero). The soft-margin rate $p$ is the fraction of distinct-real-cost
pairs satisfying $\kappa\Delta_{ij}>\eta_i+\eta_j$. Thus
$\tau_a\!\geq\!2p\!-\!1$ without real-cost ties (tie-aware form in
Supplementary Section~F). On PushT, $p$ tracks Plan-Real Spearman
with pooled Pearson $+0.895$ over $100$ pairs (Supplementary Section~F).

\paragraph{The role of inverse-dynamics training.}

\begin{claim}[Inverse-action mechanism]
\label{claim:inverse}
Inverse-action prediction encourages action-relevant latent transitions:
the latent displacement $z_{t+1} - z_t$ must contain enough information
to recover $a_t$. If the latent transition geometry is locally
well-conditioned, neither degenerate along action-relevant directions
nor anisotropically dilated relative to the data distribution, this
constraint manifests as an increased rank correlation between latent
displacement magnitude and action magnitude. Empirically, we observe
that inverse-action training increases this global rank correlation
from $-0.03$ for the LeWM baseline to $+0.38$ for inverse-only and
$+0.43$ for DA-LeWM. Inverse-only produces the largest single increment
in Plan-Real Spearman among the measured ablations.
\end{claim}

The claim is mechanistic, not a theorem. Low held-out inverse-prediction
error is an empirical diagnostic rather than an assumption of
Proposition~\ref{prop:sufficient}, and it does not force norm alignment:
action information could in principle be encoded directionally. We report
latent-displacement-vs-action-norm rank correlation in
Supplementary Section~B as a falsifiable signature of the mechanism.

\paragraph{Implications.}
Encoder distortion ($L_f/\mu_f$) and terminal rollout error
($\epsilon_H$) are the two analytical quantities. SIGReg and predictor
training are empirical interventions associated with their diagnostics,
not certificates of the constants. The inverse-action loss is instead
mediated by Claim~\ref{claim:inverse}. Reward/value decodability is
absent from the bound. We retain the exploratory all-heads checkpoint as
an implementation-specific comparison using the proxy R/V targets
described in Supplementary Section~A. It is not a general test of
reward- or value-aware world-model training.

\section{Experiments}
\label{sec:experiments}
\begin{figure*}[t]
  \centering
  \includegraphics[width=0.88\textwidth]{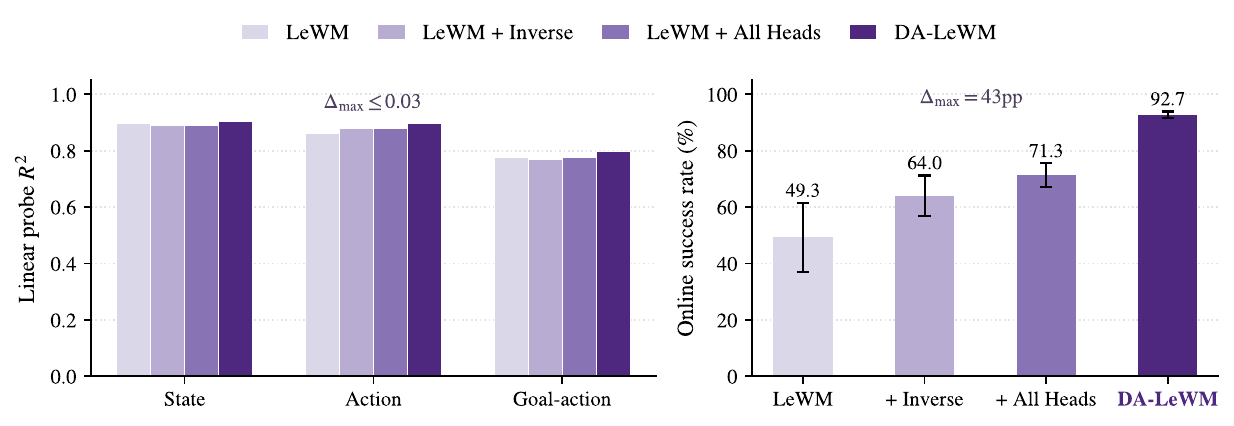}
  \caption{\emph{Probe accuracy vs.\ online success on PushT
  ($3$ evaluation seeds).} Probe scores are nearly identical across the
  four non-collapsed variants, while online success differs sharply.}
\label{fig:probe-vs-online}
\end{figure*}

\subsection{Setup}
\label{sec:exp-setup}

We evaluate on four environments: PushT, Reacher, Cube, and TwoRoom,
spanning planar object pushing, continuous-control reaching, contact-rich
manipulation, and visual navigation. Dataset and task details are provided
in Supplementary Section~A.

\paragraph{Evaluation protocol.}
All methods use matched training and planning budgets. Online success is
measured under closed-loop, receding-horizon MPC over $50$ episodes for
each of $3$ evaluation seeds. Full training, CEM, and replanning details
are provided in Supplementary Section~A.

\paragraph{Variants compared.}
We compare five variants. LeWM is the SIGReg-regularized baseline with no
auxiliary heads. No-SIGReg removes the regularizer and serves as the
collapse control. Inverse-only adds only the inverse-action head.
All-heads adds inverse, goal-action, reward-proxy, and value-proxy heads,
whereas DA-LeWM uses only the inverse and goal-action heads. We treat
all-heads as exploratory because its reward/value targets are
implementation-specific.
Supplementary Section~C reports sensitivity to the goal-action weight.
\paragraph{Evaluation sequence.}
We first use No-SIGReg to establish the collapse sanity check. We then
compare non-collapsed PushT variants across information probes, Plan-Real
Spearman, held-out action geometry, and online success. Next, CEM-stage
Spearman tracks rank agreement throughout planning. Finally, short-budget
and ten-epoch evaluations test control across tasks before
Table~\ref{tab:external-reference} places DA-LeWM alongside published
baselines.

\subsection{SIGReg Preserves Cost-Surface Variation}
\label{sec:exp-sigreg}

We audit cost-surface variation by the ratio of maximum to minimum latent
cost over a random candidate population. Removing SIGReg contracts this
ratio from $3\!-\!30\times$ to $\approx\!1.005\times$. A ratio near one
means that the planner receives almost the same objective for every plan.
Table~\ref{tab:sigreg} shows that this flattening coincides with near-zero
Plan-Real Spearman and large online-success drops on PushT, TwoRoom, and
Reacher.

\begin{table}[t]
\centering
\small
\setlength{\tabcolsep}{3pt}
\begin{tabular}{@{}lcccc@{}}
\toprule
& \multicolumn{2}{c}{Success (\%)} & \multicolumn{2}{c}{Plan-Real Sp}\\
\cmidrule(lr){2-3}\cmidrule(l){4-5}
Task & LeWM & No-SIG & LeWM & No-SIG \\
\midrule
PushT & $49.3 \pm 12.2$ & $\phantom{0}2.0 \pm 2.0$ & $+0.280$ & $+0.031$ \\
TwoRoom & $98.0 \pm \phantom{0}2.0$ & $41.3 \pm 6.1$ & $+0.549$ & $+0.012$ \\
Reacher & $82.0 \pm \phantom{0}2.0$ & $10.7 \pm 8.1$ & $+0.504$ & $+0.001$ \\
\bottomrule
\end{tabular}
\caption{\emph{Effect of removing SIGReg across three environments}
($3$ evaluation seeds, mean $\pm$ std).}
\label{tab:sigreg}
\end{table}

Cube is handled separately in Section~\ref{sec:cross-env}, because random
no-contact candidates provide too few distinct real costs for a supported
Spearman estimate.

\begin{figure*}[!t]
\centering
\includegraphics[width=0.98\textwidth]{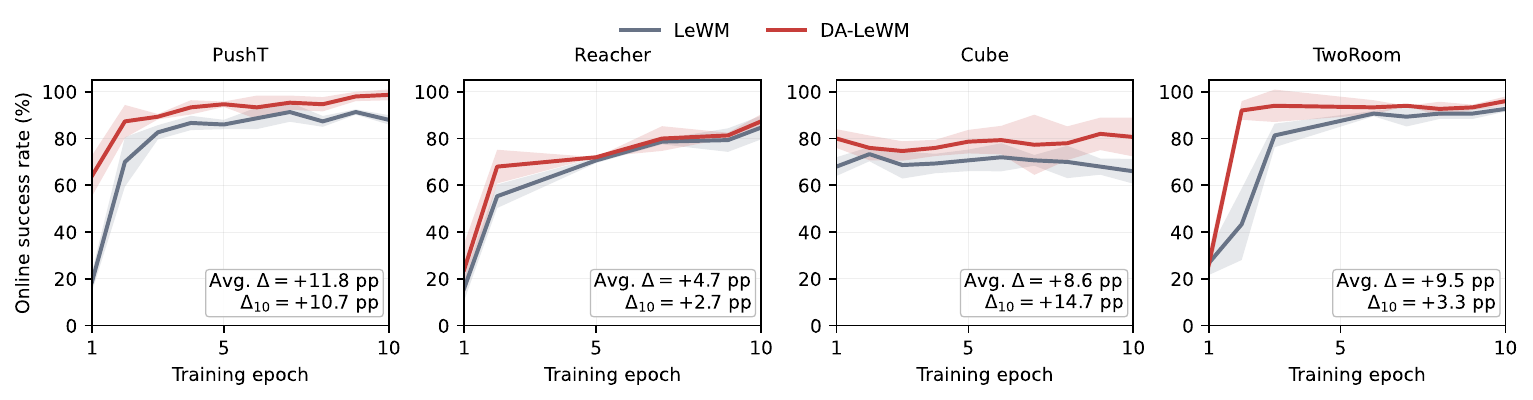}
\caption{\emph{Online success throughout ten training epochs}
($3$ evaluation seeds per task, mean $\pm$ std). Shading shows one
standard deviation. Annotations report DA-LeWM minus LeWM averaged over
epochs (Avg.\ $\Delta$) and at epoch 10 ($\Delta_{10}$).}
\label{fig:training-dynamics}
\end{figure*}

\subsection{Non-Collapsed Variants Separate Information from Planning Geometry}
\label{sec:exp-probe}

We next compare LeWM, inverse-only, all-heads, and DA-LeWM after excluding
the collapsed control. Their state, action, and goal-action probe scores vary
by at most $0.03$ in $R^2$ (Figure~\ref{fig:probe-vs-online}). In contrast,
Plan-Real Spearman changes from $+0.280$ for LeWM to $+0.410$--$+0.420$
for action-supervised variants (Table~\ref{tab:plan-real-pusht}), while online
success spans $43$ percentage points ($49.3\%\!\to\!92.7\%$). This comparison
isolates decision-metric alignment from information sufficiency rather than
conflating either property with representation collapse.
Table~\ref{tab:plan-real-pusht} retains No-SIGReg only as the collapse
reference.
For each of the $30$ held-out start--goal pairs, Spearman ranks $64$
candidates. The table reports the mean, and ``Positive pairs'' counts how
many pair-level correlations exceed zero.

\paragraph{Inverse-action mechanism.}
\label{sec:exp-inverse}

On PushT, inverse-only raises Plan-Real Spearman from $+0.280$ to $+0.420$
and online success from $49.3\%$ to $64.0\%$. DA-LeWM retains a similar
$+0.412$ Spearman and reaches $92.7\%$ success. Thus inverse-only accounts
for most of the measured ranking lift, while the combined objective is
associated with the additional online lift.

The checkpoint-level change is accompanied by the held-out transition
signature in Claim~\ref{claim:inverse}: correlation between latent- and
action-displacement magnitudes rises from $-0.03$ for LeWM to $+0.38$ for
inverse-only and $+0.43$ for DA-LeWM, versus $-0.13$ without SIGReg.
Because MPC continues to use Euclidean goal distance rather than an
auxiliary-head output, this evidence concerns the geometry exposed to the
planner, not an added test-time cost.

\begin{table}[t]
\centering
\begin{tabular}{@{}lcc@{}}
\toprule
Model & Plan-Real Sp & Positive pairs (out of $30$) \\
\midrule
LeWM & $+0.280$ & $24$ \\
No-SIGReg & $+0.031$ & $14$ \\
\textbf{Inverse-only} & $\mathbf{+0.420}$ & $30$ \\
All-heads & $+0.410$ & $30$ \\
DA-LeWM & $+0.412$ & $29$ \\
\bottomrule
\end{tabular}
\caption{\emph{Full-state Plan-Real Spearman on PushT} ($n = 30$).}
\label{tab:plan-real-pusht}
\end{table}

Supplementary Section~B gives the corresponding values for all variants.

\subsection{CEM-stage Spearman: Global Ranking Lift and Elite Saturation on PushT}
\label{sec:rv-elite}

Table~\ref{tab:cem-elite} follows CEM iterations $0/15/29$, corresponding
to the random, mid, and elite stages. The random stage measures global
candidate ordering. The later snapshots test ordering after the proposal
has adapted to the model's own cost. This stage-wise audit operationalizes
the candidate-margin term in the analysis instead of relying on a single
global correlation.

\begin{table}[t]
\centering
\begin{tabular}{@{}lccc@{}}
\toprule
Model & Random & Mid & Elite \\
\midrule
LeWM & $+0.403$ & $+0.227$ & $+0.036$ \\
No-SIGReg & $+0.029$ & $-0.017$ & $-0.102$ \\
Inverse-only & $+0.523$ & $\mathbf{+0.261}$ & $-0.089$ \\
All-heads & $+0.515$ & $+0.249$ & $-0.010$ \\
\textbf{DA-LeWM} & $\mathbf{+0.536}$ & $+0.253$ & $-0.011$ \\
\bottomrule
\end{tabular}
\caption{\emph{Full-state CEM-stage Spearman on PushT} ($n\!=\!15$
held-out pairs, CEM budget $300\!\times\!30$, top-$30$ elites).}
\label{tab:cem-elite}
\end{table}

Relative to LeWM, the three action-supervised variants produce consistent
random-stage gains of $+0.112$ to $+0.133$ (paired
$p\!\leq\!0.012$), with DA-LeWM highest at $+0.536$. These variants remain at
$+0.249$--$+0.261$ mid-stage versus $+0.227$ for LeWM. Elite-stage
correlations are near zero for every variant, indicating shared local
saturation. Supplementary Section~E reports paired significance and
local-geometry checks.
Tables~\ref{tab:plan-real-pusht} and~\ref{tab:cem-elite} therefore answer
complementary questions. The former tests independently sampled plans,
while the latter follows the distributions visited by the optimizer.

\subsection{Short-Budget Control Across Environments}
\label{sec:cross-env}

Having established the PushT ranking effect, we test whether it translates
to control under the same one-epoch budget. Table~\ref{tab:cross-env}
shows that DA-LeWM improves PushT by $43.4$ pp and Cube by $10.6$ pp,
while Reacher remains within its strong baseline band.

\begin{table}[t]
\centering
\begin{tabular}{@{}lccc@{}}
\toprule
Model & PushT & Reacher & Cube \\
\midrule
LeWM & $49.3 \pm 12.2$ & $82.0 \pm 2.0$ & $62.7 \pm 4.2$ \\
No-SIGReg & $\phantom{0}2.0 \pm \phantom{0}2.0$ & $10.7 \pm 8.1$ & $52.7 \pm 1.9$ \\
Inverse-only & $64.0 \pm \phantom{0}7.2$ & $82.7 \pm 3.1$ & $68.0 \pm 4.0$ \\
\textbf{DA-LeWM} & $\mathbf{92.7 \pm \phantom{0}1.2}$
  & $\mathbf{84.0 \pm \phantom{0}3.5}$ & $\mathbf{73.3 \pm 1.2}$ \\
\bottomrule
\end{tabular}
\caption{\emph{Cross-environment online success} (matched training
budget, $3$ evaluation seeds for all variants.)}
\label{tab:cross-env}
\end{table}

\paragraph{Diagnostic scope.}
Reacher variants stay within $[+0.49,+0.53]$, consistent with strong
baseline rank agreement. For Cube, exact no-contact ties leave too few
distinct real costs for a supported Plan-Real estimate. Its No-SIGReg
latent-cost range still contracts from $4\times$ to $1.07\times$, while
online success drops from $62.7\%$ to $52.7\%$
(Table~\ref{tab:cross-env}). We treat these as collapse and online-control
evidence, not as a Cube Spearman or global bi-Lipschitz claim. Sensitivity
to the goal-action weight $\beta$ is in Supplementary Section~C.

\subsection{Learning Dynamics under a Larger Budget}
\label{sec:long-run}

We next increase the training budget to ten epochs and compare LeWM with
DA-LeWM on all four tasks. Figure~\ref{fig:training-dynamics} reports the
resulting learning curves, with complete per-epoch statistics in
Supplementary Section~H.

The learning curves test whether the large one-epoch gains are merely an
early-training effect. DA-LeWM enters the high-success regime earlier and
maintains higher success throughout training on all four tasks. Its
advantage therefore persists as LeWM receives more optimization, supporting
accelerated convergence rather than a favorable single checkpoint.

Table~\ref{tab:external-reference} then asks whether faster learning also
yields strong absolute performance against LeWM~\citep{lewm2026},
PLDM~\citep{sobal2025pldm}, and DINO-WM~\citep{zhou2025dinowm}.
Relative to LeWM, DA-LeWM improves success by $2.7$ pp on PushT, $1.3$ pp
on Reacher, $6.7$ pp on Cube, and $9.0$ pp on TwoRoom. It ranks first on
PushT ($98.7\%$) and Reacher ($87.3\%$). On Cube, it surpasses PLDM and
remains within $5.3$ pp of DINO-WM. On TwoRoom, it reaches $96.0\%$,
within $1.0$ pp of PLDM and $4.0$ pp of DINO-WM while remaining $9.0$ pp
above LeWM. Thus the faster learning is accompanied by stronger control
performance rather than only an earlier optimization advantage.

The auxiliary heads are used only during training and are discarded at
evaluation. LeWM and DA-LeWM therefore use the same inference-time MPC
procedure and incur the same planning-time computation. Their separation
along the learning curves isolates a representation-learning benefit:
action-conditioned supervision makes the latent geometry useful to the
unchanged planner earlier in training. Together,
Figure~\ref{fig:training-dynamics} and Table~\ref{tab:external-reference}
show improvements in both learning efficiency and control performance across
task regimes.

\begin{center}
\begin{minipage}{\columnwidth}
\centering
\footnotesize
\setlength{\tabcolsep}{2.5pt}
\begin{tabular}{@{}lcccc@{}}
\toprule
Method & PushT & Reacher & Cube & TwoRoom \\
\midrule
Random & $2$ & $10$ & $48$ & $0$ \\
PLDM~\citep{sobal2025pldm} & $78$ & $78$ & $65$ & $97$ \\
DINO-WM~\citep{zhou2025dinowm} & $74$ & $79$ & $\mathbf{86}$ & $\mathbf{100}$ \\
LeWM~\citep{lewm2026} & $96$ & $86$ & $74$ & $87$ \\
\midrule
\textbf{DA-LeWM (ours)} & $\mathbf{98.7}$ & $\mathbf{87.3}$ & $80.7$ & $96.0$ \\
\bottomrule
\end{tabular}
\captionof{table}{\emph{Online success comparison (\%).} DA-LeWM values are
means over $3$ evaluation seeds.}
\label{tab:external-reference}
\end{minipage}
\end{center}
\vspace{0.5\baselineskip}

\section{Related Work}
\label{sec:related}

\paragraph{JEPA and latent world models.}
Joint-embedding predictive architectures learn visual representations by
predicting in embedding space rather than reconstructing pixels
\citep{lecun2022jepa,assran2023ijepa,bardes2024vjepa}. Masked autoencoding
and self-distillation provide complementary self-supervised routes
\citep{he2022mae,caron2021dino}. LeWM adapts joint-embedding prediction to
action-conditioned visual dynamics for manipulation, PLDM learns latent
dynamics for offline reward-free control, and DINO-WM plans over frozen
pretrained features \citep{lewm2026,sobal2025pldm,zhou2025dinowm}. These
works establish capable representations and predictors for latent planning.
Linear probes assess what is decodable \citep{alain2017probes}. Our question
is whether Euclidean distance to a goal latent induces the correct ordering
over candidate action sequences.

\paragraph{Auxiliary objectives in RL and world models.}
Inverse dynamics has been used for curiosity and action-relevant
representation learning
\citep{pathak2017curiosity,zhang2021invariant,gelada2019deepmdp}, while
contrastive and self-predictive objectives improve pixel-based RL
\citep{laskin2020curl,schwarzer2021spr}. Demonstration-conditioned
goal-action prediction relates to goal-conditioned imitation and hindsight
relabeling \citep{ding2019goal,andrychowicz2017her}. Here these established
objectives provide action-conditioned supervision. We test how that
supervision reshapes the metric structure consumed by fixed-cost latent MPC.

\paragraph{MPC and planning with learned models.}
PETS plans under learned rewards, while Dreamer learns behavior through
reward-driven latent imagination \citep{chua2018pets,hafner2019dreamer,
hafner2020dreamerv2,hafner2023dreamerv3}. MuZero and TD-MPC learn
value-equivalent dynamics, and DayDreamer extends world-model control to
physical robots \citep{schrittwieser2020muzero,hansen2022tdmpc,
hansen2024tdmpc2,wu2023daydreamer}. Our setting isolates a complementary
requirement: when Euclidean goal distance is itself the planning cost,
accurate prediction or decoding need not imply correct relative ordering.
We measure that ordering globally and as CEM concentrates its proposal.

\paragraph{Decision-aware and bisimulation-based representations.}
Bisimulation metrics ground state similarity in reward and transition
consequences \citep{ferns2004bisimulation,castro2020scalable,
zhang2021invariant}. Value equivalence provides an analogous criterion for
model predictions \citep{grimm2020value}. Plan-Real and CEM-stage Spearman
are planner-aligned empirical counterparts in this spirit: they audit the
fixed metric actually consumed by the planner rather than learning a new
task-specific metric.

\section{Conclusion}
\label{sec:conclusion}

We study when Euclidean goal distance provides a reliable objective for latent
MPC, separating information sufficiency from decision-metric alignment.
Plan-Real and CEM-stage Spearman operationalize this distinction, while the
sufficient-condition analysis identifies encoder distortion, terminal rollout
error, and candidate margins as the quantities controlling rank preservation.
Across the non-collapsed variants, similar linear-probe scores coexist with
substantial differences in Plan-Real agreement and online success, showing
that this geometric distinction matters for control. Guided by this empirical
gap, DA-LeWM uses action-conditioned supervision to improve the cost geometry
exposed to the planner. Across all four environments, DA-LeWM accelerates
convergence and achieves higher online success than LeWM.
Latent world models used for control should therefore be evaluated both for
what they encode and for the ordering induced by the planning cost they
expose.

\paragraph{Limitations.}
Our evidence is limited to four short-horizon simulated tasks, LeWM-family
checkpoints with ViT-Tiny, Euclidean goal costs, and CEM. Each configuration
has one training run, so uncertainty covers evaluation variation rather than
training-initialization variation. Plan-Real requires simulator rollouts and
loses statistical power under exact real-cost ties. We fix $\alpha=0.1$ and
test neither its sensitivity nor gradient-based planning. Generalization to
robots, partial observability, larger backbones, other planners, and learned
rewards, goal costs, or positive-semidefinite latent metrics remains open.

\paragraph{Generative-AI disclosure.}
Generative AI assisted language editing and code review. It was not a component
of the method or evaluation, and the authors verified all resulting text and
code.

\clearpage
{\small
\bibliography{references}

@unpublished{lecun2022jepa,
  author = {LeCun, Yann},
  title  = {A Path Towards Autonomous Machine Intelligence},
  note   = {Position paper, version 0.9.2, OpenReview},
  year   = {2022},
  url    = {https://openreview.net/forum?id=BZ5a1r-kVsf}
}

@inproceedings{assran2023ijepa,
  author    = {Assran, Mahmoud and Duval, Quentin and Misra, Ishan and Bojanowski, Piotr and Vincent, Pascal and Rabbat, Michael and {LeCun}, Yann and Ballas, Nicolas},
  title     = {Self-Supervised Learning from Images with a Joint-Embedding Predictive Architecture},
  booktitle = {Proceedings of the IEEE/CVF Conference on Computer Vision and Pattern Recognition (CVPR)},
  pages     = {15619--15629},
  year      = {2023}
}

@article{lewm2026,
  author  = {Maes, Lucas and Le Lidec, Quentin and Scieur, Damien and {LeCun}, Yann and Balestriero, Randall},
  title   = {{LeWorldModel}: Stable End-to-End Joint-Embedding Predictive Architecture from Pixels},
  journal = {arXiv preprint arXiv:2603.19312},
  year    = {2026},
  url     = {https://arxiv.org/abs/2603.19312}
}

@article{balestriero2025lejepa,
  author  = {Balestriero, Randall and {LeCun}, Yann},
  title   = {{LeJEPA}: Provable and Scalable Self-Supervised Learning Without the Heuristics},
  journal = {arXiv preprint arXiv:2511.08544},
  year    = {2025},
  url     = {https://arxiv.org/abs/2511.08544}
}

@inproceedings{zhou2025dinowm,
  author    = {Zhou, Gaoyue and Pan, Hengkai and {LeCun}, Yann and Pinto, Lerrel},
  title     = {{DINO-WM}: World Models on Pre-trained Visual Features Enable Zero-shot Planning},
  booktitle = {Proceedings of the 42nd International Conference on Machine Learning (ICML)},
  series    = {Proceedings of Machine Learning Research},
  volume    = {267},
  pages     = {79115--79135},
  year      = {2025},
  publisher = {PMLR},
  url       = {https://proceedings.mlr.press/v267/zhou25t.html}
}

@inproceedings{sobal2025pldm,
  author    = {Sobal, Vlad and Zhang, Wancong and Cho, Kyunghyun and Balestriero, Randall and Rudner, Tim G. J. and {LeCun}, Yann},
  title     = {Stress-Testing Offline Reward-Free Reinforcement Learning: A Case for Planning with Latent Dynamics Models},
  booktitle = {7th Robot Learning Workshop: Towards Robots with Human-Level Abilities},
  year      = {2025},
  url       = {https://openreview.net/forum?id=jON7H6A9UU}
}

@inproceedings{hafner2019dreamer,
  author    = {Hafner, Danijar and Lillicrap, Timothy and Ba, Jimmy and Norouzi, Mohammad},
  title     = {Dream to Control: Learning Behaviors by Latent Imagination},
  booktitle = {Proceedings of the 8th International Conference on Learning Representations (ICLR)},
  year      = {2020},
  url       = {https://openreview.net/forum?id=S1lOTC4tDS}
}

@inproceedings{hafner2020dreamerv2,
  author    = {Hafner, Danijar and Lillicrap, Timothy and Norouzi, Mohammad and Ba, Jimmy},
  title     = {Mastering {Atari} with Discrete World Models},
  booktitle = {Proceedings of the 9th International Conference on Learning Representations (ICLR)},
  year      = {2021},
  url       = {https://openreview.net/forum?id=0oabwyZbOu}
}

@article{hafner2023dreamerv3,
  author  = {Hafner, Danijar and Pasukonis, Jurgis and Ba, Jimmy and Lillicrap, Timothy},
  title   = {Mastering Diverse Domains through World Models},
  journal = {arXiv preprint arXiv:2301.04104},
  year    = {2023},
  url     = {https://arxiv.org/abs/2301.04104}
}

@inproceedings{chua2018pets,
  author    = {Chua, Kurtland and Calandra, Roberto and McAllister, Rowan and Levine, Sergey},
  title     = {Deep Reinforcement Learning in a Handful of Trials Using Probabilistic Dynamics Models},
  booktitle = {Advances in Neural Information Processing Systems 31 (NeurIPS)},
  year      = {2018},
  url       = {https://proceedings.neurips.cc/paper_files/paper/2018/hash/3de568f8597b94bda53149c7d7f5958c-Abstract.html}
}

@inproceedings{pathak2017curiosity,
  author    = {Pathak, Deepak and Agrawal, Pulkit and Efros, Alexei A. and Darrell, Trevor},
  title     = {Curiosity-Driven Exploration by Self-Supervised Prediction},
  booktitle = {Proceedings of the 34th International Conference on Machine Learning (ICML)},
  series    = {Proceedings of Machine Learning Research},
  volume    = {70},
  pages     = {2778--2787},
  year      = {2017},
  publisher = {PMLR}
}

@inproceedings{zhang2021invariant,
  author    = {Zhang, Amy and McAllister, Rowan and Calandra, Roberto and Gal, Yarin and Levine, Sergey},
  title     = {Learning Invariant Representations for Reinforcement Learning without Reconstruction},
  booktitle = {Proceedings of the 9th International Conference on Learning Representations (ICLR)},
  year      = {2021},
  url       = {https://openreview.net/forum?id=-2FCwDKRREu}
}

@inproceedings{ding2019goal,
  author    = {Ding, Yiming and Florensa, Carlos and Abbeel, Pieter and Phielipp, Mariano},
  title     = {Goal-Conditioned Imitation Learning},
  booktitle = {Advances in Neural Information Processing Systems 32 (NeurIPS)},
  year      = {2019},
  url       = {https://proceedings.neurips.cc/paper/2019/hash/c8d3a760ebab631565f8509d84b3b3f1-Abstract.html}
}

@inproceedings{ferns2004bisimulation,
  author    = {Ferns, Norm and Panangaden, Prakash and Precup, Doina},
  title     = {Metrics for Finite {Markov} Decision Processes},
  booktitle = {Proceedings of the 20th Conference on Uncertainty in Artificial Intelligence (UAI)},
  pages     = {162--169},
  publisher = {AUAI Press},
  year      = {2004}
}

@article{rubinstein1999cem,
  author  = {Rubinstein, Reuven Y.},
  title   = {The Cross-Entropy Method for Combinatorial and Continuous Optimization},
  journal = {Methodology and Computing in Applied Probability},
  volume  = {1},
  number  = {2},
  pages   = {127--190},
  year    = {1999},
  doi     = {10.1023/A:1010091220143}
}

@inproceedings{ha2018world,
  author    = {Ha, David and Schmidhuber, J{\"u}rgen},
  title     = {Recurrent World Models Facilitate Policy Evolution},
  booktitle = {Advances in Neural Information Processing Systems 31 (NeurIPS)},
  year      = {2018},
  url       = {https://proceedings.neurips.cc/paper/2018/hash/2de5d16682c3c35007e4e92982f1a2ba-Abstract.html}
}

@inproceedings{hafner2019planet,
  author    = {Hafner, Danijar and Lillicrap, Timothy and Fischer, Ian and Villegas, Ruben and Ha, David and Lee, Honglak and Davidson, James},
  title     = {Learning Latent Dynamics for Planning from Pixels},
  booktitle = {Proceedings of the 36th International Conference on Machine Learning (ICML)},
  series    = {Proceedings of Machine Learning Research},
  volume    = {97},
  pages     = {2555--2565},
  year      = {2019},
  publisher = {PMLR}
}

@article{schrittwieser2020muzero,
  author  = {Schrittwieser, Julian and Antonoglou, Ioannis and Hubert, Thomas and Simonyan, Karen and Sifre, Laurent and Schmitt, Simon and Guez, Arthur and Lockhart, Edward and Hassabis, Demis and Graepel, Thore and Lillicrap, Timothy and Silver, David},
  title   = {Mastering {Atari}, {Go}, Chess and {Shogi} by Planning with a Learned Model},
  journal = {Nature},
  volume  = {588},
  number  = {7839},
  pages   = {604--609},
  year    = {2020}
}

@inproceedings{hansen2022tdmpc,
  author    = {Hansen, Nicklas and Su, Hao and Wang, Xiaolong},
  title     = {Temporal Difference Learning for Model Predictive Control},
  booktitle = {Proceedings of the 39th International Conference on Machine Learning (ICML)},
  series    = {Proceedings of Machine Learning Research},
  volume    = {162},
  pages     = {8387--8406},
  year      = {2022},
  publisher = {PMLR}
}

@inproceedings{hansen2024tdmpc2,
  author    = {Hansen, Nicklas and Su, Hao and Wang, Xiaolong},
  title     = {{TD-MPC2}: Scalable, Robust World Models for Continuous Control},
  booktitle = {The Twelfth International Conference on Learning Representations (ICLR)},
  year      = {2024},
  url       = {https://openreview.net/forum?id=Oxh5CstDJU}
}

@inproceedings{schwarzer2021spr,
  author    = {Schwarzer, Max and Anand, Ankesh and Goel, Rishab and Hjelm, R. Devon and Courville, Aaron and Bachman, Philip},
  title     = {Data-Efficient Reinforcement Learning with Self-Predictive Representations},
  booktitle = {Proceedings of the 9th International Conference on Learning Representations (ICLR)},
  year      = {2021},
  url       = {https://openreview.net/forum?id=uCQfPZwRaUu}
}

@inproceedings{laskin2020curl,
  author    = {Laskin, Michael and Srinivas, Aravind and Abbeel, Pieter},
  title     = {{CURL}: Contrastive Unsupervised Representations for Reinforcement Learning},
  booktitle = {Proceedings of the 37th International Conference on Machine Learning (ICML)},
  series    = {Proceedings of Machine Learning Research},
  volume    = {119},
  pages     = {5639--5650},
  year      = {2020},
  publisher = {PMLR}
}

@inproceedings{andrychowicz2017her,
  author    = {Andrychowicz, Marcin and Wolski, Filip and Ray, Alex and Schneider, Jonas and Fong, Rachel and Welinder, Peter and McGrew, Bob and Tobin, Josh and Abbeel, Pieter and Zaremba, Wojciech},
  title     = {Hindsight Experience Replay},
  booktitle = {Advances in Neural Information Processing Systems 30 (NeurIPS)},
  year      = {2017},
  url       = {https://proceedings.neurips.cc/paper/2017/hash/453fadbd8a1a3af50a9df4df899537b5-Abstract.html}
}

@inproceedings{he2022mae,
  author    = {He, Kaiming and Chen, Xinlei and Xie, Saining and Li, Yanghao and Doll{\'a}r, Piotr and Girshick, Ross},
  title     = {Masked Autoencoders Are Scalable Vision Learners},
  booktitle = {Proceedings of the IEEE/CVF Conference on Computer Vision and Pattern Recognition (CVPR)},
  pages     = {16000--16009},
  year      = {2022}
}

@inproceedings{caron2021dino,
  author    = {Caron, Mathilde and Touvron, Hugo and Misra, Ishan and J{\'e}gou, Herv{\'e} and Mairal, Julien and Bojanowski, Piotr and Joulin, Armand},
  title     = {Emerging Properties in Self-Supervised Vision Transformers},
  booktitle = {Proceedings of the IEEE/CVF International Conference on Computer Vision (ICCV)},
  pages     = {9650--9660},
  year      = {2021}
}

@article{bardes2024vjepa,
  author  = {Bardes, Adrien and Garrido, Quentin and Ponce, Jean and Chen, Xinlei and Rabbat, Michael and {LeCun}, Yann and Assran, Mahmoud and Ballas, Nicolas},
  title   = {Revisiting Feature Prediction for Learning Visual Representations from Video},
  journal = {arXiv preprint arXiv:2404.08471},
  year    = {2024},
  url     = {https://arxiv.org/abs/2404.08471}
}

@inproceedings{gelada2019deepmdp,
  author    = {Gelada, Carles and Kumar, Saurabh and Buckman, Jacob and Nachum, Ofir and Bellemare, Marc G.},
  title     = {{DeepMDP}: Learning Continuous Latent Space Models for Representation Learning},
  booktitle = {Proceedings of the 36th International Conference on Machine Learning (ICML)},
  series    = {Proceedings of Machine Learning Research},
  volume    = {97},
  pages     = {2170--2179},
  year      = {2019},
  publisher = {PMLR}
}

@inproceedings{castro2020scalable,
  author    = {Castro, Pablo Samuel},
  title     = {Scalable Methods for Computing State Similarity in Deterministic {Markov} Decision Processes},
  booktitle = {Proceedings of the 34th AAAI Conference on Artificial Intelligence},
  pages     = {10069--10076},
  year      = {2020}
}

@inproceedings{grimm2020value,
  author    = {Grimm, Christopher and Barreto, Andr{\'e} and Singh, Satinder and Silver, David},
  title     = {The Value Equivalence Principle for Model-Based Reinforcement Learning},
  booktitle = {Advances in Neural Information Processing Systems 33 (NeurIPS)},
  year      = {2020}
}

@inproceedings{alain2017probes,
  author    = {Alain, Guillaume and Bengio, Yoshua},
  title     = {Understanding Intermediate Layers Using Linear Classifier Probes},
  booktitle = {International Conference on Learning Representations (ICLR) Workshop Track},
  year      = {2017},
  url       = {https://openreview.net/forum?id=HJ4-rAVtl}
}

@inproceedings{wu2023daydreamer,
  author    = {Wu, Philipp and Escontrela, Alejandro and Hafner, Danijar and Abbeel, Pieter and Goldberg, Ken},
  title     = {{DayDreamer}: World Models for Physical Robot Learning},
  booktitle = {Proceedings of the 6th Conference on Robot Learning (CoRL)},
  series    = {Proceedings of Machine Learning Research},
  volume    = {205},
  pages     = {2226--2240},
  year      = {2023},
  publisher = {PMLR}
}

@article{tassa2018deepmind,
  author  = {Tassa, Yuval and Doron, Yotam and Muldal, Alistair and
             Erez, Tom and Li, Yazhe and Casas, Diego de Las and
             Budden, David and Abdolmaleki, Abbas and Merel, Josh and
             Lefrancq, Andrew and Lillicrap, Timothy and Riedmiller, Martin},
  title   = {DeepMind Control Suite},
  journal = {arXiv preprint arXiv:1801.00690},
  year    = {2018}
}

@inproceedings{park2025ogbench,
  author    = {Park, Seohong and Frans, Kevin and Eysenbach, Benjamin and
               Levine, Sergey},
  title     = {{OGBench}: Benchmarking Offline Goal-Conditioned {RL}},
  booktitle = {International Conference on Learning Representations (ICLR)},
  year      = {2025}
}
}
\appendix
\section*{Appendix}

\section{Experimental Details}
\label{app:exp-details}

We provide here the full details necessary to reproduce all experiments.
The implementation is built on top of the open-source LeWM codebase
\citep{lewm2026}, with auxiliary heads and decision losses added in a
modular fashion that does not modify the base world model architecture.

\subsection{Backbone Architecture}

The encoder is a Vision Transformer (ViT-Tiny) backbone with patch size
$14$ and image resolution $224 \times 224$. The CLS token is projected
to a latent embedding of dimension $192$. The predictor is a $6$-layer
Transformer with $16$ attention heads, head dimension $64$, MLP
dimension $2048$, dropout $0.1$, and embedding dropout $0$. History
length is $3$ frames and the predictor outputs one future latent per
forward pass. All variants share this exact backbone configuration. The
only differences are in the auxiliary head set and the loss weights.

\subsection{Auxiliary Head Architecture}

Each auxiliary head is a $2$-layer MLP with hidden dimension $256$,
GELU activations, and LayerNorm, followed by a final linear layer to
the prediction target. Concretely:
\begin{itemize}
\item \textbf{Inverse-action head}: input $[z_t, z_{t+1}] \in
  \mathbb{R}^{2 \cdot 192}$, output $\hat{a}_t \in \mathbb{R}^{d_a}$.
\item \textbf{Goal-conditioned action head}: input $[z_t, z_g] \in
  \mathbb{R}^{2 \cdot 192}$, output $\hat{a}_t \in \mathbb{R}^{d_a}$.
\item \textbf{Reward proxy head}: input $[z_t, a_t]$, scalar output.
\item \textbf{Value proxy head}: input $z_t$, scalar output.
\end{itemize}
Here $d_a=5d_{\rm raw}$, where $d_{\rm raw}$ is the per-step action dimension.
For the all-heads configuration we instantiate all four heads.  The
canonical no-R/V implementation used for the PushT DA-LeWM checkpoint
instantiates only the inverse-action and goal-action heads
(\texttt{DecisionHeadsNoRV}).  For compatibility with the original
cross-environment training jobs, the Reacher, Cube, and TwoRoom
learning-curve checkpoints retain unused reward/value modules but set
both corresponding loss weights to zero.  Thus these checkpoints share
the DA-LeWM training objective, although their parameter-initialization
stream is not byte-identical to the no-R/V implementation.

\paragraph{Reward and value proxy targets.}
For the stored all-heads checkpoint, the implementation used
\begin{equation}
\begin{aligned}
r_t &= -\|s_t^{[:2]}-s_g^{[:2]}\|_2 \\
    &\quad -0.1\,d_{\angle}(s_t^{[2]},s_g^{[2]}),\\
V_t &= \sum_{k=t}^{T-1}\gamma^{k-t}r_k,\qquad \gamma=0.99.
\end{aligned}
\end{equation}
Here $d_{\angle}$ is wrapped angular distance. Neither target is an environment reward nor optimal-policy value.
Under the stored PushT state layout, these coordinates do not match the
intended block-pose fields, and no corrected all-heads checkpoint was trained.
We therefore retain this result only as an exploratory, implementation-specific ablation. It
cannot determine whether correctly specified R/V supervision helps or hurts, or explain the gap causally.

\subsection{Training Configuration}

The matched-budget ablations train every variant for one epoch.  The
learning-dynamics comparison uses matched 100-epoch training schedules
for LeWM and DA-LeWM and analyzes the first ten saved checkpoints.
Optimizer, data, and model settings are otherwise identical within each
comparison.

\begin{center}
\small
\begin{tabular}{@{}L{0.48\columnwidth}L{0.44\columnwidth}@{}}
\toprule
Hyperparameter & Value \\
\midrule
Optimizer & AdamW \\
Learning rate & $5 \times 10^{-5}$ \\
Weight decay & $10^{-3}$ \\
Batch size & $128$ \\
Gradient clipping & $1.0$ \\
Mixed precision & bf16 \\
Training schedule & $1$ epoch (ablations), $100$ epochs (first $10$ analyzed) \\
Train/val split & $0.9 / 0.1$ \\
Training randomness & Matched across variants \\
SIGReg weight $\lambda_\text{sig}$ & $0.09$ \\
SIGReg knots & $17$ \\
SIGReg projections & $1024$ \\
Inverse weight $\alpha$ & $0.1$ \\
Goal-action weight $\beta$ & $0.1$ \\
Reward weight (all-heads only) & $0.05$ \\
Value weight (all-heads only) & $0.01$ \\
Reward/value discount $\gamma$ & $0.99$ \\
Goal strategy & clip-final-state \\
\bottomrule
\end{tabular}
\end{center}

The objective-level difference between all-heads and DA-LeWM is the
presence or absence of reward/value losses. The implementation detail
above records whether zero-weight modules are retained. The base world
model loss
($\mathcal{L}_\text{pred} + \lambda_\text{sig} \mathcal{L}_\text{SIGReg}$)
is identical across all variants except No-SIGReg, where the SIGReg
term is dropped.

\subsection{Computing Infrastructure}

Training and evaluation use one NVIDIA A100-SXM4-80GB or
A800-SXM4-80GB GPU per process (80\,GB memory, driver 535.129.03) on
x86-64 Linux 5.4 cluster nodes with Intel Xeon Platinum 8362 host CPUs.
The software environment is Python 3.10.20, PyTorch 2.7.1
(\texttt{cu118}), Lightning 2.6.1, Hydra 1.3.2,
\texttt{stable-worldmodel} 0.0.6, \texttt{stable-pretraining} 0.1.6,
\texttt{dm-control} 1.0.40, MuJoCo 3.8.0, OGBench 1.2.1, NumPy 2.2.6,
SciPy 1.15.3, and scikit-learn 1.7.2.  Training uses bfloat16 mixed
precision. Diagnostics and summary statistics use float32 model
inference and double-precision host-side aggregation where required.

\subsection{Datasets}

All data are simulated trajectories from established public
environments. This work introduces no new dataset and uses no human
subjects or personally identifiable information.  PushT and TwoRoom
exercise contact-rich manipulation and obstacle-constrained
navigation, Reacher supplies a standard continuous-control test from
the DeepMind Control Suite~\citep{tassa2018deepmind}, and Cube supplies
an offline goal-conditioned manipulation task from
OGBench~\citep{park2025ogbench}.  The exact dataset identifiers are:
\begin{center}
\small
\setlength{\tabcolsep}{3pt}
\begin{tabular}{@{}L{0.17\columnwidth}L{0.47\columnwidth}L{0.28\columnwidth}@{}}
\toprule
Task & Dataset identifier & Provenance \\
\midrule
PushT & \path{pusht_expert_train} & LeWM release \\
Reacher & \path{dmc/reacher_random} & Control Suite / LeWM \\
Cube & \path{ogbench/cube_single_expert} & OGBench / LeWM \\
TwoRoom & \path{tworoom} & LeWM release \\
\bottomrule
\end{tabular}
\end{center}
Training follows the four-frame, frame-skip-5 construction in
Section~\ref{sec:method}. The 25-row goal clip below is evaluation-only.
All tasks use the same 0.9/0.1 trajectory-level train/validation split.
The code package records dataset identifiers and preprocessing.

\paragraph{Task definitions.}
The diagnostic state, raw action dimension, online success rule, and
terminal diagnostic cost are:
\begin{center}
\scriptsize
\setlength{\tabcolsep}{1.5pt}
\begin{tabular}{@{}lL{0.28\columnwidth}cL{0.43\columnwidth}@{}}
\toprule
Task & Diagnostic state / cost & $d_{\rm raw}$ & Online success \\
\midrule
PushT & agent $xy$, block $xy$/angle, agent velocity (7-D), L2 & 2 &
$\|\Delta(\text{agent }xy,\text{block }xy)\|_2<20$ and wrapped
$|\Delta\text{ block angle}|<\pi/9$ \\
TwoRoom & agent $xy$, L2 & 2 & $\|\Delta xy\|_2<16$ pixels \\
Reacher & 2-D joint position, L2 & 2 &
$\|\Delta q\|_\infty<0.05$ radians \\
Cube & cube $xyz$, L2 & 5 &
$\|\Delta xyz\|_2\leq0.04$ m \\
\bottomrule
\end{tabular}
\end{center}
Here $d_{\rm raw}$ is the per-raw-step action dimension. Each model action block
concatenates five raw actions. All listed L2 costs are unweighted,
simulator-only terminal distances and are not exposed to the planner.

\subsection{Evaluation Configuration}

Online evaluation is performed by initializing $50$ vectorized
environments per evaluation seed and running CEM-based planning. Each
reported mean aggregates $3 \times 50 = 150$ episodes. Short-budget
ablations and ten-checkpoint learning curves each use $3$ evaluation
seeds per task. Task-specific sweeps are evaluated separately and are
not pooled.

\begin{center}
\small
\begin{tabular}{@{}L{0.55\columnwidth}L{0.39\columnwidth}@{}}
\toprule
Parameter & Value \\
\midrule
Episodes per evaluation seed & $50$ \\
Short-budget evaluation repeats & $3$ seeds per task \\
Learning-curve evaluation repeats & $3$ seeds per task, task sweeps separate \\
Goal clip (\texttt{goal\_offset\_steps}) & $25$ rows (endpoint $+24$) \\
Evaluation horizon (\texttt{eval\_budget}) & $50$ raw steps \\
Planning horizon $H$ & $5$ action blocks \\
Receding-horizon stride & $5$ action blocks \\
Action block size (frame-skip) & $5$ raw actions \\
CEM samples $N$ (online) & $300$ \\
CEM iterations $K$ (online) & $30$ \\
CEM elite size $E$ (online) & $30$ \\
CEM proposal variance scale & $1.0$ \\
\bottomrule
\end{tabular}
\end{center}

The dataset loader uses the end-exclusive slice
$[\text{start},\text{start}+25)$, so the goal is the row at index
$\text{start}+24$. The five-block planning and receding horizons each
execute $25$ raw actions. These parameters are identical across
environments.
Environment-specific callables (state setters, goal setters) are
implemented as configured in the codebase and are not changed across
variants.

\subsection{Plan-Real and CEM-stage Spearman}

Diagnostic Spearman correlations use a separate evaluation procedure
that reuses the same trained checkpoints. We sample $n = 30$ (start,
goal) pairs per checkpoint from the held-out validation split. For each
pair we draw $N = 64$ candidate action sequences from $\mathcal{N}(0,
I)$ in normalized action-block space, compute latent costs from the
world model, and roll out each candidate in the environment to compute
real costs. For PushT, both Plan-Real and CEM-stage diagnostics use
$c_\text{real}=\|s_H-s_g\|_2$ on the unweighted seven-dimensional
simulator state (agent position, block position and angle, and agent
velocity), recorded as \texttt{pusht\_full\_state\_l2\_v1}.

CEM-stage Spearman uses the same CEM trajectory as the online MPC
controller ($N\!=\!300$ samples $\times$ $30$ iterations with the
top $30$ as the elite set, $n_\text{pairs}\!=\!15$) and reports
Spearman at iteration $0$ (random), iteration $15$ (mid), and the
last iteration (elite). At the elite stage each pair's Spearman is
computed over the top-$30$ candidates of the converged CEM
distribution. At the random/mid stages it is computed over $\min(N, \max(2K, 64))\!=\!64$
candidates uniformly subsampled from the $N\!=\!300$ population.
This budget is identical for every checkpoint reported in
Table~\ref{tab:cem-elite} and Appendix~\ref{app:beta-sweep}.
All Spearman calculations use average ranks for exact ties. Diagnostic
artifacts also record defined and undefined pair counts.

\subsection{Information Probes}

Linear probes are fit on frozen latents from the trained world model.
Each probe is an affine ridge regressor (closed-form least squares with
$\ell_2$ coefficient $10^{-4}$ and an intercept). No hidden layer or
nonlinear activation is used. The exact input--target pairs are:
$z_t\!\mapsto\!s_t$ (state),
$[z_t,z_{t+1}]\!\mapsto\!a_t$ (inverse action),
$[z_t,z_g]\!\mapsto\!a_t$ (goal action),
$[z_t,a_t]\!\mapsto\!r_t$ (reward proxy), and
$z_t\!\mapsto\!V_t$ (value proxy). The R/V targets are the legacy
proxies defined above. Probe $R^2$ is reported on a fixed held-out
$20\%$ split. The probe definition and split are identical across model
variants, so differences in $R^2$ reflect properties of the frozen
latent rather than probe capacity.

\section{Geometric Signature of Inverse-Action Training}
\label{app:geom-scatter}

We measure the global rank correlation between latent displacement
magnitude $\|\Delta z\|$ and action magnitude $\|a\|$ over $38{,}400$
held-out transitions. Inverse-action training changes a near-zero
baseline correlation into a moderate positive correlation, a
falsifiable empirical signature of Claim~\ref{claim:inverse}. We do
not assert this is a logical consequence of the inverse loss alone:
the inverse head can in principle recover actions through directional
not norm structure. The figure reports what we observe under the actual
training procedure.

\begin{figure*}[t]
  \centering
  \includegraphics[width=0.72\textwidth]{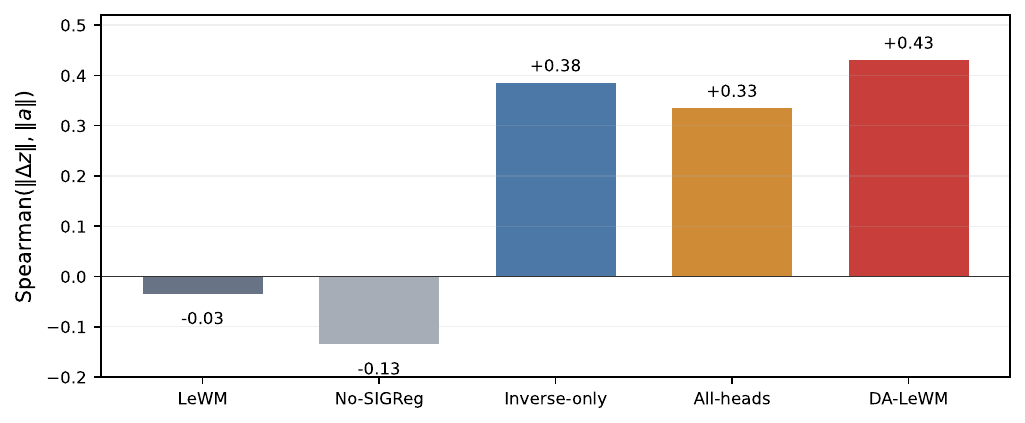}
  \caption{\emph{Global Spearman correlation between $\|\Delta z\|$
  and $\|a\|$ on $38{,}400$ held-out PushT transitions.} Values are
  read directly from the saved diagnostic output. No synthetic points
  are shown.}
  \label{fig:geom}
\end{figure*}

\section{Additional Ablations}

\subsection{Per-Method Probes and Online Success on PushT}
\label{app:probe-online-detail}

Figure~\ref{fig:probe-vs-online} (main text) plots the four
non-degenerate methods. Table~\ref{tab:probe-online} below restates
the same numbers and additionally lists the No-SIGReg degenerate
baseline as a sanity check. Probes detect the outright information
loss of No-SIGReg ($R^2$ negative or near zero) but do not separate
the four methods that train without representation collapse.

\begin{table}[t]
\centering
\small
\setlength{\tabcolsep}{3pt}
\begin{tabular}{@{}lcccc@{}}
\toprule
Model & Success & State $R^2$ & Action $R^2$ & Goal $R^2$ \\
\midrule
LeWM                 & $49.3 \pm 12.2$ & $0.90$ & $0.86$ & $0.78$ \\
No-SIGReg            & $\phantom{0}2.0 \pm \phantom{0}2.0$
                     & $-6.14$ & $0.14$ & $-0.67$ \\
Inverse-only         & $64.0 \pm \phantom{0}7.2$ & $0.89$ & $0.88$ & $0.77$ \\
All-heads            & $71.3 \pm \phantom{0}4.2$ & $0.89$ & $0.88$ & $0.78$ \\
\textbf{DA-LeWM} & $\mathbf{92.7 \pm \phantom{0}1.2}$
                     & $\mathbf{0.90}$ & $\mathbf{0.89}$ & $\mathbf{0.80}$ \\
\bottomrule
\end{tabular}
\caption{\emph{Information probes and online success on PushT
($3$ evaluation seeds).} Across the four non-collapsed methods, probe scores
change by less than $0.03$ in $R^2$ while online success spans
$43$ percentage points. No-SIGReg is shown as a degenerate sanity
check.}
\label{tab:probe-online}
\end{table}

\subsection{Goal-Action Weight Sensitivity}
\label{app:beta-sweep}

We sweep the goal-action weight $\beta$ on PushT online success and
report Plan-Real / CEM-stage Spearman for $\beta = 0.3$ alongside
$\beta = 0.1$.

\begin{center}
\begin{tabular}{@{}lc@{}}
\toprule
$\beta$ & PushT online (\%) \\
\midrule
$0.03$ & $63.3 \pm 12.2$ \\
$\mathbf{0.1}$ & $\mathbf{92.7 \pm \phantom{0}1.2}$ \\
$0.3$ & $89.3 \pm \phantom{0}3.1$ \\
\bottomrule
\end{tabular}
\end{center}

\paragraph{Spearman at $\beta = 0.3$ on PushT.} Plan-Real Spearman is
$+0.448 \pm 0.207$ ($29/30$ positive pairs), and CEM-stage Spearman
is $(\text{random}, \text{mid}, \text{elite}) = (+0.549, +0.232,
-0.038)$. Relative to LeWM, the Plan-Real and random-stage gains are
significant (paired $p\!=\!0.002$ for both), whereas mid and elite are
not ($p\!=\!0.90/0.33$).

\paragraph{Additional same-budget initialization.}
A separately initialized $\beta=0.1$ checkpoint gives Plan-Real
Spearman $+0.436 \pm 0.224$ ($29/30$ positive pairs) and CEM-stage
Spearman $(+0.542,+0.237,+0.028)$ at random/mid/elite. Its Plan-Real
and random-stage gains over LeWM are significant (paired
$p\!<\!0.001$ and $p\!=\!0.002$), whereas mid and elite are not
($p\!=\!0.85/0.91$). Together, the two additional checks reproduce
global alignment improvement and shared local saturation.

\paragraph{Cross-environment behavior.} On Reacher, $\beta = 0.3$
regresses to $78.0\%$ (vs.\ $84.0\%$ at $\beta = 0.1$ and $82.0\%$
LeWM baseline). Reacher per-variant Spearman numbers including
$\beta = 0.3$ are reported in Appendix~\ref{app:cross-env-sp}. The
Reacher regression motivates the conservative cross-environment
default $\beta = 0.1$.

\section{Scope of the Sufficient-Condition Diagnostics}
\label{app:diagnostics}

Proposition~\ref{prop:sufficient} is a sufficient-condition analysis,
not a claim that its constants are identified by the training losses.
In particular, empirical average prediction loss does not establish
the pointwise terminal-rollout constant $\epsilon_H$, and finite held-out pairs
cannot certify global bi-Lipschitz constants $\mu_f$ and $L_f$.
We therefore do not attach numerical estimates to those constants.
Instead, the paper tests observable implications at three levels:
SIGReg collapse via latent-cost dynamic range and tie-aware Plan-Real
Spearman on PushT, TwoRoom, and Reacher (main text), action-conditioned local geometry via
Figure~\ref{fig:geom}, and the rank-preservation implication via the
soft-margin check in Supplementary Section~F.  These measurements are
diagnostics consistent with the analysis, not proofs that its
assumptions hold out of distribution.

\section{Elite-Neighborhood Local Latent Geometry}
\label{app:elite-geom}

This appendix details the elite-neighborhood diagnostic referenced in
Section~\ref{sec:rv-elite}. The diagnostic was originally designed to
discriminate between two candidate mechanisms for an apparent elite-stage
Spearman gap:
(i) gradients from the R/V losses distort the local latent metric in
the elite neighborhood (\emph{gradient interference}), or
(ii) R/V heads encode global task-progress scalars that are
uninformative once candidates are near-optimal but do not distort
local geometry (\emph{global-progress encoding}).
After the same-budget re-evaluation in Table~\ref{tab:cem-elite}
and Supplementary Section~C showed that elite-stage Spearman is
uniformly weak across the five main variants and two additional DA-LeWM
checks (with paired tests vs LeWM all non-significant), the elite-stage
\emph{gap} no longer requires a mechanistic explanation. The
diagnostic results below remain useful: they provide direct evidence
that all four head configurations share the same weakly-informative
local latent metric in the elite neighborhood
($\mathrm{Pearson}(\log d^z, \log d^s) \approx 0.08$ across variants),
which independently corroborates the elite-stage uniformity reported
in Table~\ref{tab:cem-elite}.

\paragraph{Procedure.} For each of $n\!=\!15$ held-out
$(\text{start},\text{goal})$ pairs from the PushT validation set we
run CEM with $300$ candidates, $30$ iterations and elite size
$E\!=\!K\!=\!30$ (matching the planner setting) and retain the
final-iteration elite set. For every elite candidate $i$ we record
\begin{itemize}
\item $z_i \in \mathbb{R}^{192}$, the predicted final-step embedding
  produced by the JEPA rollout (the latent the planner scores against
  $z_g$), and
\item $s_i \in \mathbb{R}^7$, the real PushT final-step state after
  rolling out the action sequence in the simulator.
\end{itemize}
We then compute pairwise distances within the elite set,
$d^z_{ij} = \|z_i - z_j\|_2$ and
$d^s_{ij} = \|s_i - s_j\|_\Sigma$, with
$\Sigma = \mathrm{diag}(\sigma^{-2})$ and $\sigma$ the per-coordinate
state standard deviation on the training set (diagonal Mahalanobis). The local geometric ratio is
$\rho_{ij} = d^z_{ij} / d^s_{ij}$. A locally isotropic latent has
roughly constant $\rho_{ij}$. Gradient interference would produce
direction-dependent contraction with elevated $\rho_{ij}$ variance.

\paragraph{Results.} Table~\ref{tab:elite-geom} reports per-pair
statistics aggregated over the $n\!=\!15$ pairs.

\begin{table}[ht!]
\centering
\small
\setlength{\tabcolsep}{3pt}
\begin{tabular}{@{}lccc@{}}
\toprule
Variant & $\rho$ mean & $\rho$ CV & Log spread \\
\midrule
LeWM (no DS)              & $2.13 \pm 0.86$ & $0.710 \pm 0.200$ & $1.87 \pm 0.26$ \\
Inverse-only              & $2.90 \pm 1.08$ & $0.714 \pm 0.172$ & $2.03 \pm 0.36$ \\
All-heads                 & $2.92 \pm 1.20$ & $0.685 \pm 0.116$ & $1.98 \pm 0.41$ \\
DA-LeWM ($\beta\!=\!0.1$) & $2.90 \pm 1.04$ & $0.767 \pm 0.118$ & $2.12 \pm 0.20$ \\
\bottomrule
\end{tabular}

\vspace{3pt}
\begin{tabular}{@{}lc@{}}
\toprule
Variant & $\mathrm{Pearson}(\log d^z,\log d^s)$ \\
\midrule
LeWM (no DS)              & $+0.092 \pm 0.067$ \\
Inverse-only              & $+0.080 \pm 0.084$ \\
All-heads                 & $+0.084 \pm 0.065$ \\
DA-LeWM ($\beta\!=\!0.1$) & $+0.050 \pm 0.066$ \\
\bottomrule
\end{tabular}
\caption{\emph{Elite-neighborhood local latent geometry on PushT}
($K\!=\!30$ elites/pair, $n\!=\!15$ pairs). Mean $\pm$ std across
pairs. CV is $\sigma(\rho)/\mu(\rho)$. Log spread is
$\log(p_{95}/p_5)$.}
\label{tab:elite-geom}
\end{table}

Across the four variants, the coefficient of variation spans
$0.685$--$0.767$ and log spread spans $1.87$--$2.12$ in the elite
neighborhood. In particular, all-heads does not differ significantly
from inverse-only, DA-LeWM, or LeWM on either statistic. This provides
no evidence for the R/V-specific gradient-interference hypothesis:
if the proxy R/V losses produced a detectable increase in local
anisotropy, all-heads should differ from the R/V-free configurations.
Table~\ref{tab:elite-geom-paired} reports the paired contrasts.

\begin{table}[ht!]
\centering
\small
\setlength{\tabcolsep}{3pt}
\begin{tabular}{@{}lcc@{}}
\toprule
Contrast & $\Delta\mathrm{CV}\ (t)$ & $\Delta$ log spread $(t)$ \\
\midrule
All-heads    $-$ Inverse-only & $-0.029\ (-0.73)$ & $-0.049\ (-0.87)$ \\
All-heads    $-$ DA-LeWM      & $-0.081\ (-1.85)$ & $-0.135\ (-1.09)$ \\
All-heads    $-$ LeWM         & $-0.024\ (-0.47)$ & $+0.119\ (+0.85)$ \\
DA-LeWM      $-$ Inverse-only & $+0.053\ (+0.99)$ & $+0.086\ (+0.87)$ \\
DA-LeWM      $-$ LeWM         & $+0.057\ (+0.93)$ & $+0.255\ (+2.50)$ \\
Inverse-only $-$ LeWM         & $+0.004\ (+0.06)$ & $+0.169\ (+1.34)$ \\
\bottomrule
\end{tabular}
\caption{\emph{Paired comparisons of elite-neighborhood local
geometry on PushT} ($n\!=\!15$ pairs, same pair indices across
variants). Positive numbers indicate the first variant has higher
anisotropy. Parentheses report paired $t$. All all-heads contrasts
are non-significant. Among the remaining contrasts, only DA-LeWM versus
LeWM log spread is significant ($p=0.026$).}
\label{tab:elite-geom-paired}
\end{table}

The Pearson correlation
$\mathrm{Pearson}(\log d^z_{ij}, \log d^s_{ij})$
is uniformly low ($+0.050$--$+0.092$) across all four variants,
showing weak within-elite association between latent and real-state
distance regardless of which auxiliary heads are present. This is
consistent with the full-state elite-stage results in
Table~\ref{tab:cem-elite}: elite correlations remain near zero and the
paired test against LeWM is non-significant for every main
configuration. The isolated DA-LeWM--LeWM log-spread contrast does not
support an R/V-specific effect because all three contrasts involving
the all-heads checkpoint are non-significant.

\paragraph{Implementation.} The diagnostic is implemented in
\texttt{decision/eval\_elite\_local\_geometry.py} and re-uses the CEM
inner loop and dataset / scaler / env wiring of the behavioral
Plan-Real Spearman pipeline. Total wall-clock for the four variants
on a single 80\,GB A100/A800 GPU is $\approx\!9$ minutes
($n\!=\!15$ pairs,
$30$ CEM iterations $\times$ $300$ samples per iteration), so the
diagnostic is essentially free relative to the CEM-stage Spearman
evaluation it complements.

\section{Soft-Margin Consistency Check for Corollary~\ref{cor:rank}}
\label{app:corollary-check}

\paragraph{Procedure.} Implementation:
\texttt{decision/eval\_corollary\_margin\_check.py}. For each PushT
validation pair we sample $N\!=\!64$ random plans (matching the
Plan-Real Spearman protocol), compute $c_\text{lat}^{(i)}$ via JEPA
and the full-state PushT cost $c_\text{real}^{(i)}=\|s_H^{(i)}-s_g\|_2$
via environment rollout, fit
$\kappa\!=\!\sum_i c_\text{lat}^{(i)} c_\text{real}^{(i)} /
\sum_i (c_\text{real}^{(i)})^2$ (OLS through origin), and define
residuals $\eta_i\!=\!|c_\text{lat}^{(i)}\!-\!\kappa c_\text{real}^{(i)}|$.
The soft margin condition is
$\kappa\,\Delta_{ij}\!>\!\eta_i\!+\!\eta_j$ for the ordered pair with
$c_\text{real}^{(i)}\!<\!c_\text{real}^{(j)}$.
We define $p$ as the fraction of such distinct-real-cost pairs satisfying
the condition. Kendall's $\tau_a$ is the mean concordance sign over all
unordered candidate pairs, with tied pairs contributing zero. If $q$ is
the fraction of pairs with distinct real costs, the tie-aware bound is
$\tau_a\geq q(2p-1)$. Here $q=1$ for the reported full-state PushT pairs.

\paragraph{Results.} Table~\ref{tab:corollary-check} reports per-variant
means over $n\!=\!20$ pairs.

\begin{table}[ht!]
\centering
\small
\setlength{\tabcolsep}{1.25pt}
\begin{tabular}{@{}lccc@{}}
\toprule
Variant & $p$ & $\rho_s$ & $\tau_a$ \\
\midrule
LeWM (no DS)              & $0.288 \pm 0.073$ & $+0.338 \pm 0.206$ & $+0.239 \pm 0.148$ \\
No-SIGReg                 & $0.221 \pm 0.052$ & $+0.053 \pm 0.209$ & $+0.037 \pm 0.142$ \\
Inverse-only              & $0.314 \pm 0.088$ & $+0.448 \pm 0.207$ & $+0.324 \pm 0.154$ \\
All-heads                 & $0.317 \pm 0.093$ & $+0.457 \pm 0.236$ & $+0.331 \pm 0.178$ \\
DA ($\beta\!=\!0.1$)      & $0.325 \pm 0.091$ & $+0.479 \pm 0.228$ & $+0.347 \pm 0.174$ \\
\bottomrule
\end{tabular}

\vspace{3pt}
\begin{tabular}{@{}lcc@{}}
\toprule
Variant & $2p\!-\!1$ & $\kappa$ \\
\midrule
LeWM (no DS)              & $-0.424 \pm 0.146$ & $+1.93$ \\
No-SIGReg                 & $-0.557 \pm 0.103$ & $+0.003$ \\
Inverse-only              & $-0.371 \pm 0.176$ & $+2.28$ \\
All-heads                 & $-0.367 \pm 0.187$ & $+2.32$ \\
DA ($\beta\!=\!0.1$)      & $-0.351 \pm 0.182$ & $+2.34$ \\
\bottomrule
\end{tabular}
\caption{\emph{Corollary~\ref{cor:rank} soft-margin consistency check on
PushT} ($N\!=\!64$ random plans/pair, $n\!=\!20$ pairs, mean
$\pm$ std across pairs). $\tau_a \geq 2p-1$ holds for every pair.}
\label{tab:corollary-check}
\end{table}

Across the five variant means, $p$ tracks $\rho_s$ with Pearson
correlation $+0.9996$ (Spearman rank $+1.0$). Pooling all
$5\!\times\!20\!=\!100$ (variant, pair) points gives
Pearson$(p,\rho_s)=+0.895$ ($p=3.54\!\times\!10^{-36}$). The Kendall
lower bound $\tau_a\!\geq\!2p\!-\!1$ holds for all $100$ pairs. The
empirical gap $\tau_a-(2p-1)$ is $0.670\pm0.083$. Thus the soft condition
tracks ranking consistently while remaining a conservative, loose
lower bound.

\section{Cross-Environment Spearman Scope}
\label{app:cross-env-sp}

This appendix collects the behavioral Plan-Real Spearman and CEM-stage
Spearman measurements on Reacher and explains why the same summary is
not meaningful on contact-sparse Cube, supporting the cross-environment
discussion in Section~\ref{sec:cross-env}.
Table~\ref{tab:plan-real-pusht} and Table~\ref{tab:cem-elite} in the
main body cover PushT in detail.

\subsection{Reacher}

Reacher exhibits a high baseline alignment (LeWM $\mathrm{Sp} =
+0.504$). Auxiliary objectives leave behavioral Spearman essentially
unchanged, consistent with limited empirical headroom in this
diagnostic. These measurements do not estimate the constants or
tightness of Proposition~\ref{prop:sufficient}.

\begin{center}
\small
\setlength{\tabcolsep}{3pt}
\begin{tabular}{@{}lcccc@{}}
\toprule
Variant & Behavior & Random & Mid & Elite \\
\midrule
LeWM & $+0.504 \pm 0.244$ & $+0.468$ & $+0.446$ & $+0.220$ \\
No-SIGReg & $+0.001 \pm 0.135$ & $-0.015$ & $-0.018$ & $-0.045$ \\
Inverse-only & $+0.506 \pm 0.254$ & $+0.486$ & $+0.457$ & $+0.111$ \\
DA ($\beta\!=\!0.1$) & $+0.523 \pm 0.217$ & $+0.492$ & $+0.460$ & $\mathbf{+0.231}$ \\
DA ($\beta\!=\!0.3$) & $+0.514 \pm 0.242$ & $+0.486$ & $+0.467$ & $+0.097$ \\
\bottomrule
\end{tabular}
\end{center}

The variation across alignment-improving variants is within sampling
noise ($n = 30$ pairs, std $\sim 0.22$), and online success similarly
varies by less than $3$ percentage points across the same variants.

\subsection{Cube: Tie-Dominated Real Costs}

Random Cube plans often fail to contact the object. Consequently, many
candidates have exactly the same terminal object position and real
cost. Some $64$-candidate populations contain a tie block of size $63$,
and fully constant real-cost vectors also occur in both behavioral and
CEM-stage populations. Average-rank Spearman is undefined for a
constant vector and is low-support when only one or two distinct costs
remain.

We therefore do not report Cube Spearman as a cross-variant metric.
The appropriate conclusion is narrower: random-plan rank agreement is
not informative in this contact-sparse regime. Cube comparisons in the
paper use online success, while the tied-cost frequency is itself
reported as a limitation of Plan-Real and CEM-stage diagnostics.

\section{Training Dynamics and Checkpoint Selection}
\label{app:training-dynamics}

We evaluate checkpoints from epochs 1 through 10 of matched 100-epoch
schedules for LeWM and DA-LeWM. Training randomness is matched across
the two methods, and online evaluation uses $50$ episodes for each of
$3$ evaluation seeds per task. Task-specific sweeps are evaluated
separately and are not pooled. Intermediate checkpoints are
non-monotonic, especially on Reacher and Cube. We therefore use epoch
10 as a common, pre-specified endpoint and do not select a separate
best epoch for each method. The complete per-epoch means and sample
standard deviations underlying Figure~\ref{fig:training-dynamics} are
included in the accompanying results file.

\end{document}